\documentclass{elsarticle}
\usepackage{amsmath}
\usepackage{times}
\usepackage{indentfirst}
\usepackage{graphicx}
\usepackage{subfigure}
\usepackage{amssymb}
\newtheorem{theorem}{Theorem}
\newtheorem{definition}{Definition}
\newtheorem{proposition}{Proposition}
\newtheorem{corollary}{Corollary}
\newtheorem{lemma}{Lemma}
\newtheorem{example}{Example}
\newdefinition{remark}{Remark}
\newproof{proof}{Proof}

\begin{document}

\begin{frontmatter}

\title{Geometric lattice structure of covering and its application to attribute reduction through matroids}
\author{Aiping Huang}
\author{William Zhu}
\ead{williamfengzhu@gmail.com}
\address{Lab of Granular Computing,\\
Minnan Normal University, Zhangzhou 363000, China}

\begin{abstract}
The reduction of covering decision systems is an important problem in data mining, and covering-based rough sets serve as an efficient technique to process the problem.
Geometric lattices have been widely used in many fields, especially greedy algorithm design which plays an important role in the reduction problems.
Therefore, it is meaningful to combine coverings with geometric lattices to solve the optimization problems.
In this paper, we obtain geometric lattices from coverings through matroids and then apply them to the issue of attribute reduction.
First, a geometric lattice structure of a covering is constructed through transversal matroids.
Then its atoms are studied and used to describe the lattice.
Second, considering that all the closed sets of a finite matroid form a geometric lattice, we propose a dependence space through matroids and study the attribute reduction issues of the space, which realizes the application of geometric lattices to attribute reduction.
Furthermore, a special type of information system is taken as an example to illustrate the application.
In a word, this work points out an interesting view, namely, geometric lattice to study the attribute reduction issues of information systems.
\end{abstract}


\end{frontmatter}

\section{Introduction}
Rough set theory~\cite{Pawlak82Rough}, based on equivalence relations, was proposed by Pawlak to deal with the vagueness and incompleteness of knowledge in information systems.
It has been widely applied to many practical applications in various areas, such as attribute reductions~\cite{FanZhu12Attribute,MinHeQianZhu11Test,WeiLiZhang07Knowledge}, rule extractions~\cite{DuHuZhuMa11Rule} and so on.
In order to extend rough set theory's applications, some scholars have extended the theory to generalized rough set theory based on tolerance relation~\cite{YaoWangZhang10Onfuzzy}, similarity relation~\cite{SlowinskiVanderpooten00AGeneralized} and arbitrary binary relation~\cite{LiuZhu08TheAlgebraic,Zhu09RelationshipBetween}.
Through extending a partition to a covering, rough set theory has been generalized to covering-based rough set theory~\cite{zakowski83Approximations,Zhu09RelationshipBetween}.
Because of its high efficiency in many complicated problems such as knowledge reduction and rule learning in incomplete information system, covering-based rough set theory has been attracting increasing research interest~\cite{QinGaoPei07OnCovering,WangZhuZhu10Structure,WangZhu12Quantitative,Zhu11Covering,Zhu07Topological,ZhuWang03Reduction,WangChenHeHu12A,WangChenSunHu12Communication}.

A lattice is suggested by the form of the Hasse diagram depicting it.
In mathematics, a lattice is a partially ordered set in which any two elements have a unique supremum (also called a least upper bound or join) and a unique infimum (also called a greatest lower bound or meet). They encode the algebraic behavior of the entailment relation and such basic logical connectives as ``and" (conjunction) and ``or"(disjunction), which results in adequate algebraic semantics for a variety of logical systems.
Lattices, especially geometric lattices, are important algebraic structures and are used extensively in both theoretical and
applicable fields, such as rough sets~\cite{EstajiHooshmandaslDavva12Rough,HuangZhu12Geometric}, formal concept analysis~\cite{WangLiuCao10ANew,Wille82Restructuring,YaoChen04Rough} and domain theory~\cite{Birhoff95Lattice,DaveyPriestley90Introduction}.

Matroid theory~\cite{Lai01Matroid,Oxley93Matroid} borrows extensively from linear algebra and graph theory.
There are dozens of equivalent ways to characterize a matroid.
Significant definitions of a matroid include those in terms of independent sets, bases, circuits, closed sets or flats and rank functions, which provides well-established platforms to connect with other theories.
In applications, matroids have been widely used in many fields such as combinatorial optimization, network flows and algorithm design, especially greedy algorithm design~\cite{Edmonds71Matroids,Lawler01Combinatorialoptimization}.
Studying rough sets with matroids is helpful to enrich the theory system and to extend the applications of rough sets.
Some works on the connection between rough sets and matroids have been conducted~\cite{HuangZhu13Dependence,HuangZhaoZhu13Nullity,LiLiu12Matroidal,TangSheZhu12Matroidal,TangSheMinZhu13AMatroidal,WangZhuMin11Transversal,
WangZhu11Matroidal,WangZhuZhuMin12Matroidal,ZhuWang11Matroidal,ZhuWang13Rough}.

In this paper, we pay our attention to geometric lattice structures of coverings and their applications to attribute reduction issues of information systems.
First, a geometric lattice of a covering is constructed through the transversal matroid induced by the covering.
Then its atoms are studied and used to characterize the lattice structure.
It is interesting that any element of the lattice can be expressed as the union of all closures of single-point sets in the element.
Second, we apply the obtained geometric lattice to attribute reduction issues in information systems.
It is interesting that a subset of a finite nonempty set is a reduct of the information system if and only if it is a minimal set with respect to the property of containing an element from each nonempty complement of any coatom of the lattice.

The rest of this paper is organized as follows.
In Section \ref{Preliminaries}, we recall some fundamental concepts related to rough sets, lattices and matroids.
Section \ref{S:Geometriclatticestructureofcoveringthroughmatroids} presents a geometric lattice of a covering, and characterizes the structure by the atoms of the lattice.
In Section \ref{S:Applicationofgeometriclatticeinattributereduction}, we apply the obtained geometric lattices to the attribute reduction issues of information systems.
Finally, this paper is concluded and further work is pointed out in Section \ref{S:conclusions}.

\section{Preliminaries}
\label{Preliminaries}

In this section, we review some basic concepts of rough sets, matroids, and geometric lattices.

\subsection{Rough sets}

Rough set theory is a new mathematical tool for imprecise and incomplete data analysis.
It uses equivalence relations (resp.partitions) to describe the knowledge we can master.
In this subsection, we introduce some concepts of rough sets used in this paper.

\begin{definition}(Covering and partition)
Let $U$ be a universe, $\mathcal{C}$ a family of subsets of $U$.
If none of subsets in $\mathcal{C}$ are empty and $\bigcup \mathcal{C} = U$, then $\mathcal{C}$ is called a covering of $U$.
The element of $\mathcal{C}$ is called a covering block.
If $\mathcal{P}$ is a covering of $U$ and it is a family of pairwise disjoint subsets of $U$, then $P$ is called a partition of $U$.
\end{definition}

It is clear that a partition is certainly a covering, so the concept of a covering is an extension of the concept of a partition.

\begin{definition}(Approximation operators~\cite{Pawlak82Rough})
Let $U$ be a finite set and $R$ be an equivalent relation (reflexive, symmetric and transitive) on $U$.
$\forall X \subseteq U$, the lower and upper approximations of $X$, are, respectively, defined as follows:
\begin{center}
~~~~~~~~~~~~~~~~~$R_{\ast}(X) =\{x \in U: [x]_{R} \subseteq X\}$,\\
~~~~~~~~~~~~~~~~~~~~$R^{\ast}(X) = \{x \in U: [x]_{R} \bigcap X \neq \emptyset\}$,
\end{center}
where $[x]_{R}$ is called the equivalence class of $x$ with respect to $R$.
\end{definition}

\subsection{Matroids}

Matroid theory borrows extensively from the terminology of linear algebra and graph theory,
largely because it is the abstraction of various notions of central importance in these fields, such as independent sets, bases and the rank function.

\begin{definition}(Matroid~\cite{Oxley93Matroid})
A matroid is an ordered pair $(U, \mathcal{I})$ consisting of a finite set $U$ and a collection $\mathcal{I}$ of subsets of $U$ satisfying the following three conditions:\\
(I1) $\emptyset \in \mathcal{I}$.\\
(I2) If $I \in \mathcal{I}$ and $I^{'} \subseteq I$, then $I^{'} \in \mathcal{I}$.\\
(I3) If $I_{1}, I_{2} \in \mathcal{I}$ and $|I_{1}| < |I_{2}|$, then there is an element $e \in I_{2} - I_{1}$ such that $I_{1} \bigcup \{e\} \in \mathcal{I}$, where $|X|$ denotes the cardinality of $X$.
\end{definition}

Let $M = (U, \mathcal{I})$ be a matroid.
The members of $\mathcal{I}$ are the independent sets of $M$.
A set in $\mathcal{I}$ is maximal, in the sense of inclusion, is called a base of $M$.
If $X \notin \mathcal{I}$, $X$ is called a dependent set of $M$.
In the sense of inclusion, a minimal dependent subset of $U$ is called a circuit of $M$.
The collections of the bases and the circuits of matroid $M$ are denoted by $\mathcal{B}(M)$ and $\mathcal{C}(M)$, respectively.
The rank function of matroid $M$ is a function $r_{M}: 2^{U} \rightarrow N$ defined by $r_{M}(X) = max\{|I|: I \subseteq X, I \in \mathcal{I}\}$, where $X \subseteq U$.
For each $X \subseteq U$, we say $cl_{M}(X) = \{a \in U: r_{M}(X) = r_{M}(X \bigcup \{a\})\}$ is the closure of $X$ in $M$.
If $cl_{M}(X) = X$, $X$ is called a closed set of $M$.
For any $X \subseteq U$, if $cl_{M}(X) = X$ and $r_{M}(X) = r_{M}(U) - 1$, then $X$ is called a hyperplane in $M$.
The rank function of a matroid, directly analogous to a similar theorem of linear algebra, has the following proposition.

\begin{proposition}(Rank axiom~\cite{Oxley93Matroid})
\label{P:Rankaxiom}
Let $U$ be a set. A function $r:2^{U} \rightarrow N$ is the rank function of a matroid on $U$ if and only if it satisfies the following conditions:\\
(R1) For all $X \subseteq U$, $0 \leq r(X) \leq |X|$.\\
(R2) If $X \subseteq Y \subseteq U$, then $r(X) \leq  r(Y)$.\\
(R3) If $X, Y \subseteq U$, then $r(X \bigcup Y) + r(X \bigcap Y)\leq r(X) + r(Y)$.
\end{proposition}

The following proposition is the closure axiom of a matroid.
It means that an operator satisfies the following four conditions if and only if it is the closure operator of a matroid.

\begin{proposition}(Closure axiom~\cite{Oxley93Matroid})
\label{P:Closureaxiom}
Let $U$ be a set. A function $cl:2^{U} \rightarrow 2^{U}$ is the closure operator of a matroid on $U$ if and only if it satisfies the following conditions:\\
(1) If $X \subseteq U$, then $X\subseteq cl(X)$.\\
(2) If $X \subseteq Y \subseteq U$, then $cl(X) \subseteq cl(Y)$.\\
(3) If $X \subseteq U$, $cl(cl(X))=cl(X)$.\\
(4) If $X \subseteq U, x \in U$ and $y \in cl(X \bigcup \{x\}) - cl(X)$, then $x \in cl(X \bigcup \{y\})$.
\end{proposition}

Transversal theory is a branch of matroids. It shows how to induce a matroid, namely, transversal matroid from a family of subsets of a set.
Hence, transversal matroids establish a bridge between a collection of subsets of a set and a matroid.

\begin{definition}(Transversal~\cite{Oxley93Matroid})
Let $S$ be a nonempty finite set and $J = \{1, 2, \cdots, m\}$.
$\mathcal{F}= \{F_{1}, F_{2}, \cdots, F_{m}\}$ denotes a family of subsets of $S$.
A transversal or system of distinct representatives of $\{F_{1}, F_{2}, \cdots, F_{m}\}$ is a subset $\{e_{1}, e_{2}, \cdots, e_{m}\}$ of $S$ such that $e_{i} \in F_{i}$ for all $i \in J$.
If for a subset $K$ of $J$, $X$ is a transversal of $\{F_{i}: i \in K\}$, then $X$ is called a partial transversal of $\{F_{1}, F_{2}, \cdots, F_{m}\}$.
\end{definition}

\begin{example}
\label{E:example1}
Let $S = \{1, 2, 3, 4, 5\}$, $F_{1} = \{1, 3\}, F_{2} = \{2, 3\}$ and $F_{3} = \{3, 4\}$.
For $\mathcal{F} = \{F_{1}, F_{2}, F_{3}\}$, $T=\{2, 3, 4\}$ is a transversal of $\mathcal{F}$ because $2 \in F_{2}$, $3 \in F_{1}$ and $4 \in F_{3}$. $T^{'} = \{2, 4\}$ is a partial transversal of $\mathcal{F}$ because there exists a subset of $\mathcal{F}$, i.e., $\{F_{2}, F_{3}\}$, such that $T^{'}$ is a transversal of it.
\end{example}

The following proposition shows what kind of matroid is a transversal matroid.

\begin{proposition}(Transversal matroid~\cite{Oxley93Matroid})
\label{P:Transversalmatroid}
Let $\mathcal{F} = \{F_{i}: i \in J\}$ be a family of subsets of $U$.
$M(\mathcal{F}) = (U, \mathcal{I}(\mathcal{F}))$ is a matroid, where $\mathcal{I}(\mathcal{F})$ is the family of all partial transversals of $\mathcal{F}$.
We call $M(\mathcal{F}) = (U,\mathcal{I}(\mathcal{F}))$ the transversal matroid induced by $\mathcal{F}$.
\end{proposition}

\begin{example}(Continued from Example \ref{E:example1})
$\mathcal{I}(\mathcal{F})=\{\{1\}, \{2\}, \{3\}, \{4\}, \{1, 2\}, \{1, 3\}, $ $\{1, 4\}, \{2, 3\}, \{2, 4\}, \{3, 4\}, \{1, 2, 3\}, \{1, 2, 4\}, \{1, 3, 4\}, \{2, 3, 4\}\}$.
\end{example}

%
%
\subsection{Geometric lattice}
A lattice is a poset $(\mathcal{L}, \leq)$ such that, for every pair of elements, the least upper bound and greatest lower bound of the pair exist.
Formally, if $x$ and $y$ are arbitrary elements of $\mathcal{L}$, then $\mathcal{L}$ contains elements $x \bigvee y$ and $x \bigwedge y$.
The element $a$ of $\mathcal{L}$ is an atom of lattice $(\mathcal{L}, \leq)$ if it satisfies the condition: $0 < a$ and there is no $x \in \mathcal{L}$ such that $0 < x < a$.
The element $a$ of $\mathcal{L}$ is a coatom of lattice $(\mathcal{L}, \leq)$ if it satisfies the condition: $a < 1$ and there is no $x \in \mathcal{L}$ such that $a < x < 1$.
The following lemma gives another definition of a geometric lattice from the viewpoint of matroids.
In fact, the set of all closed sets of a matroid, ordered by inclusion, is a geometric lattice.

\begin{proposition}~\cite{Oxley93Matroid}
A lattice $\mathcal{L}$ is a geometric lattice if and only if it is the lattice of closed sets of a matroid.
\end{proposition}

The above proposition indicates that $(\mathcal{L}(M), \subseteq )$ is a geometric lattice, where $\mathcal{L}(M)$ denotes the collection of all closed sets of matroid $M$.
The operations join and meet of the lattice are, respectively, defined as $X \bigwedge Y = X \bigcap Y$ and $X \bigvee Y = cl_{M}(X \bigcup Y)$ for all $X , Y \in \mathcal{L}(M)$.
Moreover, the height of any element of the lattice is equal to the rank of the element in $M$.
As we know, the atoms of a lattice are precisely the elements of height one.
Therefore, the collection of the atoms of the lattice is the family of the sets which are closed sets of matroid $M$ and have value $1$ as their ranks.

\section{Geometric lattice structure of covering through matroids}
\label{S:Geometriclatticestructureofcoveringthroughmatroids}

As we know, a collection of all the closed sets of a matroid, in the sense of inclusion, is a geometric lattice.
In this section, we convert a covering to a matroid through transversal matroids, then study the lattice of all the closed sets of the matroid.
By this way, we realize the purpose to construct a geometric lattice structure from a covering.

Let $U$ be a nonempty and finite set and $\mathcal{F}$ a collection of nonempty subsets of $U$.
As shown in Proposition \ref{P:Transversalmatroid}, $M(\mathcal{F})$ is the transversal matroid induced by $\mathcal{F}$ and we denote the geometric lattice of $\mathcal{F}$ by $(\mathcal{L}(M(\mathcal{F})), \subseteq )$.
When $\mathcal{F}$ is a covering $\mathcal{C}$, the geometric lattice corresponding to it is denoted by $(\mathcal{L}(M(\mathcal{C})), \subseteq)$.
For convenience, we substitute $x$ for $\{x\}$ in the following discussion.

\subsection{Atoms of the geometric lattice structure induced by a covering}
\label{Sub:Atomsofgeometriclatticeofacovering}
Atoms of a geometric lattice are elements that are minimal among the non-zero elements and can be used to express the lattice.
Therefore, atoms play an important role in the lattices.
In this subsection, we study the atoms of the geometric lattice structure induced by a covering.

A covering of universe of objects is the collection of some basic knowledge we master, therefore it is important to be studied in detail.
The following theorem provides some equivalence characterizations for a covering from the viewpoint of matroids.

\begin{lemma}~\cite{Oxley93Matroid}
\label{L:therankofasetanditscclosedset}
Let $M$ be a matroid of $U$ and $X \subseteq U$.
$r_{M}(X) = r_{M}(cl_{M}(X))$.
\end{lemma}

\begin{lemma}~\cite{Oxley93Matroid}
\label{L:aconditionfortwosetshavethesameclosedsets}
Let $M$ be a matroid of $U$ and $X, Y \subseteq U$.
If $X \subseteq Y$ and $r_{M}(X) = r_{M}(Y)$, then $cl_{M}(X) = cl_{M}(Y)$.
\end{lemma}

\begin{theorem}
\label{T:equivalentcharacteizationforcovering}
Let $\mathcal{F}$ be a family of nonempty subsets of $U$ and $\mathcal{F} \neq \emptyset$.
The following statements are equivalent.\\
(1): $\mathcal{F}$ is a covering of $U$.\\
(2): $cl_{\mathcal{F}}(\emptyset) = \emptyset$.\\
(3): $\{cl_{\mathcal{F}}(x): x \in U\}$ is a partition of $U$.\\
(4): $\{cl_{\mathcal{F}}(x): x \in U\}$ is the collection of the atoms of $(\mathcal{L}(M(\mathcal{F})), \subseteq)$.
\end{theorem}

\begin{proof}
``$(1) \Rightarrow (2)$": According to the definition of transversal matroids, any partial transversal is an independent set.
Since $\mathcal{F}$ is a covering, any single-point set is an independent set.
Therefore, $cl_{M(\mathcal{F})}(\emptyset)=\emptyset$.

``$(2) \Rightarrow (4)$": For all $x \in U$, $cl_{M(\mathcal{F})}(cl_{M(\mathcal{F})}(x)) = cl_{M(\mathcal{F})}(x)$, then $cl_{M(\mathcal{F})}(x) \in \mathcal{L}(M(\mathcal{F}))$.
Since $cl_{M(\mathcal{F})}(\emptyset) = \emptyset$, any single-point set is an independent set, that is, $\forall x \in U$, $r_{M(\mathcal{F})}(x) = 1$.
Utilizing Lemma \ref{L:therankofasetanditscclosedset}, we have $r_{M(\mathcal{F})}(cl_{M(\mathcal{F})}(x)) = r_{M(\mathcal{F})}(x) = 1$.
Thus, for all $x \in U$, $cl_{M(\mathcal{F})}(x)$ is an atom of the lattice $(\mathcal{L}(M(\mathcal{F})), \subseteq)$.
Conversely, if $A$ is an atom of the lattice $(\mathcal{L}(M(\mathcal{F})), \subseteq)$, then $cl_{M(\mathcal{F})}(A) = A$ and $r_{M(\mathcal{F})}(A) = 1$.
It is clear that $A \neq \emptyset$.
Pitch $x \in A$, then $r_{M(\mathcal{F})}(x) = 1 = r_{M(\mathcal{F})}(A)$.
Utilizing Lemma \ref{L:aconditionfortwosetshavethesameclosedsets}, we have $cl_{M(\mathcal{F})}(x) = cl_{M(\mathcal{F})}(A) = A$.
Therefore, we have proved the result.

``$(4) \Rightarrow (3)$":
We firstly prove: $\forall x, y \in U$, if $cl_{M(\mathcal{F})}(x) \bigcap cl_{M(\mathcal{F})}(y) \neq \emptyset$, then $cl_{M(\mathcal{F})}(x) = cl_{M(\mathcal{F})}(y)$.
We may as well suppose $z \in cl_{M(\mathcal{F})}(x) \bigcap cl_{M(\mathcal{F})}(y)$.
Then $cl_{M(\mathcal{F})}(\emptyset) \subseteq cl_{M(\mathcal{F})}(z) \subseteq cl_{M(\mathcal{F})}(x)$ and $cl_{M(\mathcal{F})}(\emptyset) \subseteq cl_{M(\mathcal{F})}(z) \subseteq cl_{M(\mathcal{F})}(y)$.
We conclude that $cl_{M(\mathcal{F})}(\emptyset) \neq cl_{M(\mathcal{F})}(z)$;
Otherwise, $r_{M(\mathcal{F})}(cl_{M(\mathcal{F})}(z))$ $=$ $r_{M(\mathcal{F})}(cl_{M(\mathcal{F})}(\emptyset))$ $=$ $r_{M(\mathcal{F})}(\emptyset) = 0$, which contradicts that $cl_{M(\mathcal{F})}(z)$ is an atom.
Thus $cl_{M(\mathcal{F})}(\emptyset) \subset cl_{M(\mathcal{F})}(z) \subseteq cl_{M(\mathcal{F})}(x)$ and $cl_{M(\mathcal{F})}(\emptyset) \subset cl_{M(\mathcal{F})}(z) \subseteq cl_{M(\mathcal{F})}(y)$.
According to the definition of atoms, we have $cl_{M(\mathcal{F})}(x) = cl_{M(\mathcal{F})}(z) = cl_{M(\mathcal{F})}(y)$.
$\forall x \in U$, $x \in cl_{M(\mathcal{F})}(x)$, then $cl_{M(\mathcal{F})}(x) \neq \emptyset$.
Moreover, $U = \bigcup_{x\in U} \{x\} \subseteq \bigcup_{x\in U} cl_{M(\mathcal{F})}(x) \subseteq \bigcup_{x \in U} U = U$, that is, $U = \bigcup_{x \in U} cl_{M(\mathcal{F})}(x)$.
Hence $\{cl_{M(\mathcal{F})}(x): x \in U\}$ is a partition of $U$.

``$(3) \Rightarrow (1)$":
Since $\{cl_{\mathcal{F}}(x): x \in U\}$ is a partition of $U$, for any distinct elements $cl_{M(\mathcal{F})}(u)$ and $cl_{M(\mathcal{F})}(v)$ of $\{cl_{M(\mathcal{F})}(x): x \in U\}$, $cl_{M(\mathcal{F})}(u) \bigcap cl_{M(\mathcal{F})}(v) = \emptyset$.
Thus $cl_{M(\mathcal{F})}(\emptyset) \subseteq cl_{M(\mathcal{F})}(u)$ and $cl_{M(\mathcal{F})}(\emptyset) \subseteq cl_{M(\mathcal{F})}(v)$.
Hence $cl_{M(\mathcal{F})}(\emptyset) \subseteq cl_{M(\mathcal{F})}(u) \bigcap cl_{M(\mathcal{F})}(v) = \emptyset$.
Combining with $\emptyset \subseteq cl_{M(\mathcal{F})}(\emptyset)$, we have $cl_{M(\mathcal{F})}(\emptyset) = \emptyset$.
Thus, for all $x \in U$, $x$ is an independent set, that is, there exists $F_{x} \in \mathcal{F}$ such that $x \in F_{x}$.
Hence, $U = \bigcup_{x\in U}\{x\} \subseteq \bigcup_{x\in U} F_{x} \subseteq \bigcup \mathcal{F} \subseteq U$.
Thus $\bigcup \mathcal{F} = U$.
Since $\mathcal{F}$ is a family of nonempty subsets of $U$ and $\mathcal{F} \neq \emptyset$, we know $\emptyset \notin F$.
Therefore, we have proved that $\mathcal{F}$ is a covering of $U$.
\end{proof}

Theorem \ref{T:equivalentcharacteizationforcovering} indicates that the closure of any single-point set is an atom of lattice $\mathcal{L}(M(\mathcal{C}))$.
Based on the fact, we obtain all the atoms of the lattice from covering $C$ directly.

\begin{definition}~\cite{HuangZhu12Geometric}
\label{D:atoms}
Let $\mathcal{C} = \{K_{1}, K_{2}, \cdots, K_{m}\}$ be a covering of a finite set $U = \{x_{1}, x_{2}, \cdots, x_{n}\}$.
We define\\
(i) $A = \{K_{i} - \bigcup_{j=1, j \neq i }^{m}K_{j}: K_{i} - \bigcup_{j=1,j\neq i }^{m} K_{j} \neq \emptyset, i\in \{1, 2, \cdots, m\}\}\\
~~~~~~~~~= \{A_{1}, A_{2}, \cdots, A_{s}\}$.\\
(ii) $B = U - \bigcup_{i=1}^{s} A_{i}$.
\end{definition}

\begin{remark}
For all $i\in \{1, 2, \cdots, s\}$ and $x\in A_{i}$, there exists only one block of covering $C$ such that $x$ belongs to it, and there exist at least two blocks of covering $C$ such that $y$ belongs to them for all $y\in B$.
\end{remark}

One example is provided to illustrate the above definition.

\begin{example}
\label{E:example3}
Let $U = \{1, 2, 3, 4, 5\}$ and $\mathcal{C} = \{K_{1}, K_{2}, K_{3}\}$, where $K_{1} = \{1, 3\}, K_{2}$ $ = \{2, 3\}, K_{3} = \{3, 4, 5\}$.
Then $K_{1} - K_{2} \bigcup K_{3} = \{1, 3\} - \{2, 3, 4, 5\} = \{1\}$,
Similarly, $K_{2} = \{2\}$ and $K_{3} = \{4, 5\}$.
Therefore $A = \{\{1\}, \{2\}, \{4, 5\}\}$ and $B = U - \bigcup A = U - \{1, 2, 4, 5\} = \{3\}$.
\end{example}

In fact, the closure of any singleton set of universe $U$ in matroid $M(\mathcal{C})$ is an element of $A$ or a single-point set of $B$.
In order to reveal the fact better, we need the two lemmas below.

\begin{lemma}
\label{L:therelationbetweencoveringblockandclosureofsingle-pointset}
Let $\mathcal{C}$ be a covering of $U$.
For all $x \in U$, there exists $K \in \mathcal{C}$ such that $cl_{M(\mathcal{C})}(x) \subseteq K$.
\end{lemma}

\begin{proof}
For all $x \in U$, we take $K \in \mathcal{C}$ satisfies $x \in K$, then $cl_{M(\mathcal{C})}(x) \subseteq K$.
In fact, for all $y \notin K$, there exists $K^{'} \neq K$ such that $y \in K^{'}$ because $\mathcal{C}$ is a covering of $U$.
That implies that $\{x, y\}$ is an independent set because there exist $K, K^{'} \in \mathcal{C}$ such that $x \in K$ and $y \in K^{'}$.
Thus, $r_{M(\mathcal{C})}(\{x, y\}) = 2 \neq 1 = r_{M(\mathcal{C})}(x)$ which implies that $y \notin cl_{M(\mathcal{C})}(x)$.
Therefore, we prove the result.
\end{proof}

\begin{lemma}
\label{L:therelationbetweencoveringblockandclosureofsingle-pointsetunderacondition}
Let $\mathcal{C}$ be a covering of $U$.
For all $x \in U$, if $|cl_{M(\mathcal{C})}(x)| \geq 2$, then there exists only one block $K$ of $\mathcal{C}$ such that $cl_{M(\mathcal{C})}(x) \subseteq K$.
\end{lemma}

\begin{proof}
According to Lemma \ref{L:therelationbetweencoveringblockandclosureofsingle-pointset}, we know there exists $K \in \mathcal{C}$ such that $cl_{M(\mathcal{C})}(x) \subseteq K$ for all $x \in U$.
Now, we need to prove the uniqueness of $K$.
Suppose there exists the other block $K^{'}$ such that $cl_{M(\mathcal{C})}(x) \subseteq K^{'}$.
We claim that $cl_{M(\mathcal{C})}(x) = \{x\}$; otherwise, there exists $y \neq x$ such that $y \in cl_{M(\mathcal{C})}(x)$ because we have had $x \in cl_{M(\mathcal{C})}(x)$.
That implies that $r_{M(\mathcal{F})}(\{x, y\}) = r_{M(\mathcal{C})}(x) = 1$.
However, there exist two blocks $K$ and $K^{'}$ such that $cl_{M(\mathcal{C})}(x)$ is contained in them.
Thus $r_{M(\mathcal{C})}(\{x,y\}) = 2$, which implies a contradiction!
Hence $cl_{M(\mathcal{C})}(x) = \{x\}$, i.e., $|cl_{M(\mathcal{C})}(x)| = 1$ which contradicts the assumption $|cl_{M(\mathcal{C})}(x)| \geq 2$.
\end{proof}

\begin{proposition}
\label{P:equivalencecharacterizationforatomsofclosurelattice}
Let $\mathcal{C}$ be a covering of $U$.
$\{cl_{M(\mathcal{C})}(x): x \in U\} = \{A_{1}, A_{2}, $ $\cdots, A_{s}\}$ $ \bigcup \{\{x\}: x \in B\}$.
\end{proposition}

\begin{proof}
For all $cl_{M(\mathcal{C})}(x) \in \{cl_{M(\mathcal{C})}(x): x \in U\}$, if $|cl_{M(\mathcal{C})}(x)| = 1$, then $cl_{M(\mathcal{C})}(x)$ $ = \{x\}$ because $x \in cl_{M(\mathcal{C})}(x)$.
Since $\mathcal{C}$ is a covering, there exists a block $K$ of $\mathcal{C}$ such that $x \in K$.
If $K$ is a unique block of $\mathcal{C}$ such that $x \in K$, then there exists $A_{i} \in A$ such that $x \in A_{i}$.
Moreover, $A_{i} = \{x\}$; otherwise, there exists $y \neq x$ such that $y \in A_{i}$.
According to the definition of $A_{i}$, we have $r_{M(\mathcal{C})}(\{x, y\}) = r_{M(\mathcal{C})}(x)$, i.e, $y \in cl_{M(\mathcal{C})}(x)$ which contradicts the assumption $|cl_{M(\mathcal{C})}(x)| = 1$.
Hence $cl_{M(\mathcal{C})}(x) = \{x\} = A_{i}$.
If $K$ is not a unique block of $\mathcal{C}$ such that $x \in K$, then $x \notin A_{i}$ for all $i \in \{1, 2, \cdots, s\}$.
That implies that  $x \in B$.
Therefore, $cl_{M(\mathcal{C})}(x) = \{x\}$, where $x \in B$.
If $|cl_{M(\mathcal{C})}(x)| \neq 1$, then $|cl_{M(\mathcal{C})}(x)| \geq 2$.
According to Lemma \ref{L:therelationbetweencoveringblockandclosureofsingle-pointsetunderacondition}, we know there exists only one block $K_{i}$ such that $cl_{M(\mathcal{C})}(x) \subseteq K_{i}$.
According to the definition of $A$, there exists $A_{j} \in A$ such that $x \in cl_{M(\mathcal{C})}(x) \subseteq A_{j}$.
$\forall y \in A_{j}$ and $y \neq x$, we know $r_{M(\mathcal{C})}(\{x, y\}) = r_{M(\mathcal{C})}(x)$.
Thus $y \in cl_{M(\mathcal{C})}(x)$, that is, $A_{j} \subseteq cl_{M(\mathcal{C})}(x)$, therefore, $A_{j} = cl_{M(\mathcal{C})}(x)$.
Form above discussion, we conclude that $\{cl_{M(\mathcal{C})}(x): x \in U \} \subseteq \{A_{1}, A_{2}, $ $\cdots, A_{s}\} \bigcup \{\{x\}: x \in B\}$.

Next, we prove $\{A_{1}, A_{2}, $ $\cdots, A_{s}\} \bigcup \{\{x\}: x \in B\} \subseteq \{cl_{M(\mathcal{C})}(x): x \in U \}$.
For all $A_{i} \in \{A_{1}, A_{2}, $ $\cdots, A_{s}\}$, we know there exists a unique block $K \in \mathcal{C}$ such that $A_{i} \subseteq K$.
Thus $r_{M(\mathcal{C})}(A_{i}) = 1$.
Pitch $y \in A_{i}$.
Since $\mathcal{C}$ is a covering, $r_{M(\mathcal{C})}(y) = 1$.
Thus $r_{M(\mathcal{C})}(A_{i}) = r_{M(\mathcal{C})}(y)$.
Utilizing Lemma \ref{L:aconditionfortwosetshavethesameclosedsets}, we have $cl_{M(\mathcal{C})}(y) = cl_{M(\mathcal{C})}(A_{i})$, which implies that $A_{i} \subseteq cl_{M(\mathcal{C})}(y)$.
For all $x \in cl_{M(\mathcal{C})}(y)$, $r_{M(\mathcal{C})}(\{x, y\}) = 1$, i.e., there exists a unique block $K$ of $\mathcal{C}$ such that $\{x,y\} \subseteq K$.
Thus $x \in A_{i}$.
Therefore $cl_{M(\mathcal{C})}(y) \subseteq A_{i}$ which implies that $\{A_{1}, A_{2}, $ $\cdots, A_{s}\} \subseteq \{cl_{M(\mathcal{C})}(x): x \in U \}$.
$\forall x \in B$, we claim that $cl_{M(\mathcal{C})}(x) = \{x\}$.
Since $\{x\} \subseteq cl_{M(\mathcal{C})}(x)$, we just need to prove $cl_{M(\mathcal{C})}(x) \subseteq \{x\}$; otherwise, there exists $y \in U$ and $y \neq x$ such that $y \in cl_{M(\mathcal{C})}(x)$.
Utilizing Lemma \ref{L:therelationbetweencoveringblockandclosureofsingle-pointsetunderacondition}, there is only one block $K$ of $\mathcal{C}$ such that $\{x, y\} \subseteq K$.
According to the definition of $A_{i}(i \in \{1, 2, \cdots, s\})$, we know there exists $A_{j}(j \in \{1, 2, \cdots, s\})$ such that $x \in A_{j}$, thus $x \notin B$ which contradicts the assumption that $x \in B$.
\end{proof}

The following result is the combination of Theorem \ref{T:equivalentcharacteizationforcovering} and Proposition \ref{P:equivalencecharacterizationforatomsofclosurelattice}.
It presents the atoms of lattice $(\mathcal{L}(M(\mathcal{C})), \subseteq)$ from covering $C$ directly.

\begin{corollary}
\label{C:concreteformofatiomofclosedsetlatticeofacovering}
Let $\mathcal{C}$ be a covering of $U$.
$\{A_{1}, A_{2}, \cdots, A_{s}\} \bigcup \{\{x\}: x\in B\}$ is the family of atoms of lattice $(\mathcal{L}(M(\mathcal{C})), \subseteq)$.
\end{corollary}

Corollary \ref{C:concreteformofatiomofclosedsetlatticeofacovering} can also be found in \cite{HuangZhu12Geometric}.
It provides a method to obtain the atoms of the geometric lattice induced by a covering from the covering directly.
We obtain the result from the other different perspective in this paper.

\begin{example}(Continued from Example \ref{E:example3})
\label{E:exampel4}
Based on Corollary \ref{C:concreteformofatiomofclosedsetlatticeofacovering}, the collection of the atoms of lattice $(\mathcal{L}(M(\mathcal{C})), \subseteq)$ is $\{\{1\}, \{2\}, \{3\}, \{4, 5\}\}$.
\end{example}

\subsection{Atoms characterization for the geometric lattice induced by a covering}
\label{sub:Atomscharacterizationforgeometriclatticeofacovering}

In subsection \ref{Sub:Atomsofgeometriclatticeofacovering}, we have studied the atoms of the geometric lattice induced by a covering and have provided a method to obtain the atoms from the covering directly.
As we know, any element of a geometric lattice can be expressed as the joint of some atoms of the lattice.
In this subsection, we characterize the geometric lattice induced by a covering through the atoms of it by the union operation.
In fact, any element of the lattice can be indicated as the union of all closures of single-point sets in the element.
At the beginning of this subsection, we define two operators from the viewpoint of matroids.

\begin{definition}
Let $M$ be a matroid on $U$ and $X \subseteq U$.
One can define the following two operators:
\begin{center}
$L_{M}(X) = \{x \in U: cl_{M}(x) \subseteq X\}$,\\
$~~~~~~~H_{M}(X) = \{x \in U: cl_{M}(x) \bigcap X \neq \emptyset\}$.
\end{center}
We call the two operators are lower and upper approximation operators induced by $M$.
\end{definition}

In fact, $cl_{M}(x)$ can be regard as the successor neighborhood of $x$ with respect to the relation $R$ defined as $xR_{M}y \Leftrightarrow y \in cl_{M}(x)$.
It is clear that $R$ is a reflexive and transitive relation.
When $M = M(\mathcal{C})$, the relation $R_{M(\mathcal{C})}$ is an equivalence relation on $U$.
Therefore, $L_{M(\mathcal{C})} = (R_{M(\mathcal{C})})_{\ast}$ and $H_{M(\mathcal{C})} = (R_{M(\mathcal{C})})^{\ast}$.
In the following discussion, we study the relationship between the two operators and the elements of the lattice $\mathcal{L}(M(\mathcal{C}))$.
Then, based on the relationship, we realize the purpose to characterize the lattice through the atoms of it by using union operator.
Firstly, we have the following lemma.

\begin{lemma}
\label{L:upperdefinablesetandlowerdefinableaboutpartition}
Let $R$ be an equivalence relation of $U$.
For all $X \subseteq U$, if $R_{\ast}(X) = X$, then $R^{\ast}(X) = X$.
\end{lemma}

\begin{proof}
It is clear that $X = \{x \in U: [x]_{R} \subseteq X\} \subseteq \{x \in U: [x]_{R} \bigcap X \neq \emptyset\} = R^{\ast}(X)$.
For all $y \in R^{\ast}(X)$, $[y]_{R} \bigcap X \neq \emptyset$.
Suppose $z \in [y]_{R} \bigcap X$.
Then $y \in [y]_{R} = [z]_{R} \subseteq X$, hence $R^{\ast}(X) \subseteq X$.
\end{proof}

In fact, any closed set of the matroid induced by a covering is a fixed point of the two operators
induced by the covering.

\begin{proposition}
\label{P:aexpressionofclosedset}
Let $\mathcal{C}$ be a covering of $U$.
If $X \in \mathcal{L}(M(\mathcal{C}))$, then $L_{M(\mathcal{C})}(X) = X = H_{M(\mathcal{C})}(X)$.
\end{proposition}

\begin{proof}
Utilizing Lemma \ref{L:upperdefinablesetandlowerdefinableaboutpartition}, we need prove $X = L_{M(\mathcal{C})}(X)$. 
For all $y \in L_{M(\mathcal{C})}(X)$, $y \in cl_{M(\mathcal{C})}(y) \subseteq X$, thus $L_{M(\mathcal{C})}(X) \subseteq X$.
Conversely, according to (2) of Proposition \ref{P:Closureaxiom}, we know for all $y \in X$, $cl_{M(\mathcal{C})}(y)\subseteq X$.
Thus $X \subseteq L_{M(\mathcal{C})}(X)$.
Hence $X = L_{M(\mathcal{C})}(X)$.
\end{proof}

Based on the above result, any element of the geometric lattice induced by a covering can be expressed as the union of all closures of single-point sets in the element.

\begin{theorem}
Let $\mathcal{C}$ be a covering of $U$.
$ \forall X \in \mathcal{L}(M(\mathcal{C}))$, $X = \bigcup_{x \in X}cl_{M(\mathcal{C})}(x)$.
\end{theorem}

\begin{proof}
It is obvious when $X = \emptyset$.
According to Proposition \ref{P:aexpressionofclosedset}, we have $X = \{x \in U: cl_{M(\mathcal{C})}(x) \subseteq X\}$.
Then $x \in cl_{M(\mathcal{C})}(x) \subseteq X$ for all $x \in X$.
Thus $X = \bigcup_{x \in X}\{x\} \subseteq \bigcup_{x \in X} cl_{M(\mathcal{C})}(x) \subseteq X$.
Therefore $X = \bigcup_{x \in X} cl_{M(\mathcal{C})}(x)$.
\end{proof}

\begin{example}
\label{E:example5}
Suppose $\mathcal{C}$ is the one shown in Example \ref{E:example3}.
Combining with Example \ref{E:exampel4} and Proposition \ref{P:equivalencecharacterizationforatomsofclosurelattice}, we have $cl_{M(\mathcal{C})}(\emptyset) = \emptyset, cl_{M(\mathcal{C})}(1)$ $ = \{1\}, cl_{M(\mathcal{C})}(2) = \{2\}, cl_{M(\mathcal{C})}(3) = \{3\}$ and $cl_{M(\mathcal{C})}(4) = cl_{M(\mathcal{C})}(5) $ $= \{4, 5\}$.
Since $X = \bigcup_{x \in X} cl_{M(\mathcal{C})}(x)$ for all $X \in \mathcal{L}(M(\mathcal{C}))$.
Then we obtain $\mathcal{L}(M(\mathcal{C})) = \{\emptyset, \{1\}, \{2\}, \{3\}, \{4,$ $ 5\}, \{1, 2\},$ $ \{1, 3\}, $ $\{1, 4, 5\}, $ $\{2, 3\},$ $ \{2, 4, 5\}, \{3, 4, 5\}, \{1, 2, 3, 4,$ $ 5\}\}$, and the geometric lattice $(\mathcal{L}(M(\mathcal{C})),$ $ \subseteq)$ is shown in Figure 1.
\begin{figure}[htp]
   \begin{center}
   \includegraphics[width=3.5in]{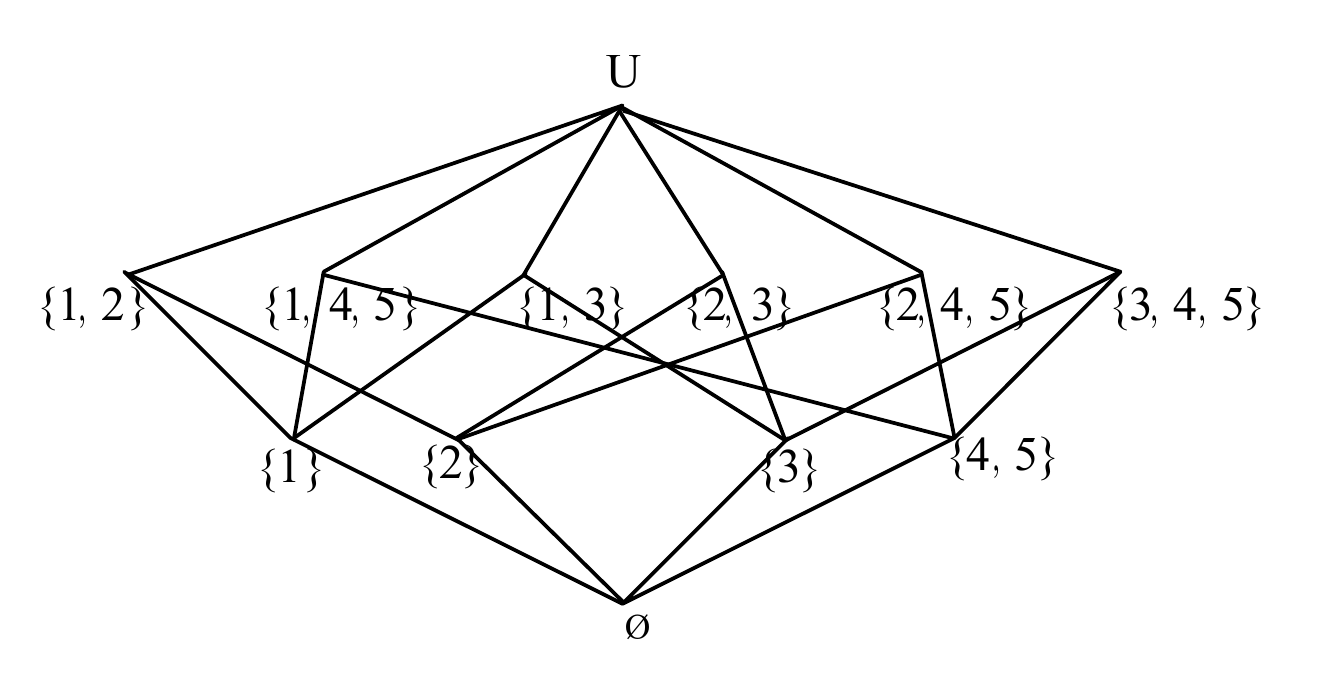}
   \caption{The geometric lattice of $(\mathcal{L}(M(\mathcal{C})), \subseteq)$}.
    \end{center}
\end{figure}
\end{example}

\section{Application of geometric lattice in attribute reduction}
\label{S:Applicationofgeometriclatticeinattributereduction}

In section \ref{S:Geometriclatticestructureofcoveringthroughmatroids}, we have studied the geometric lattice structure induced by a covering in detail.
In this section, we study how to apply the lattice to attribute reductions from an expanded perspective.
Considering the fact that an information system can be converted to a dependence space and studying the reduction issues of the dependence space is equal to studying the issues of the information system, and the fact that a geometric lattice is the lattice
of all the closed sets of a finite matroid, hence we take the following measures to
realize our purpose. First, we construct one dependence space through a matroid and obtain all the reducts of the space.
Second, we built the other dependence space from an information system.
Through making these two spaces are equal, we realize the purpose to apply geometric lattices to the issues of attribute reduction of information
systems.

\subsection{Application of geometric lattice in the reduction issue of dependence space}
\label{Sub:Geometriclatticeandmatroidalapproachestoreductionissuesindependencespaces}

In this subsection, we apply the geometric lattices to the reduction problems of dependence spaces.
First, we make certain what is dependence space.
The concept of dependence space can be found in \cite{ZhangLiangWu03Information}, the following lemma introduces it from the viewpoint of set theory.

\begin{lemma}\cite{ZhangLiangWu03Information}
\label{L:dependencespace}
Let $U$ be a finite nonempty set.
For all $\mathcal{T} \subseteq 2^{U}$, denote
\begin{center}
    $\Gamma(\mathcal{T}) = \{(B_{1}, B_{2}) \in P(U) \times P(U): B_{1} \subseteq X \Leftrightarrow B_{2} \subseteq X, \forall X \in \mathcal{T}\}$.
\end{center}
Then $(U, \Gamma(\mathcal{T}))$ is a dependence space.
\end{lemma}

For a geometric lattice induced by a matroid, one can use its coatoms, namely the hyperplanes of the matroid to induce a dependence space $(U, \Gamma(\mathcal{H}(M)))$.
Before studying the reduction issues of the dependence space, we review the concepts of consistent sets and reducts defined in dependence spaces.

Let $(U, \Theta)$ be a dependence space.
A subset $B(\subseteq U)$ is called a consistent set, if $B$ is minimal with respect to inclusion in its $\Theta-$ class.
If $B$ is called a reduct of $U$, if $B \Theta U$ and $B$ is a consistent set.

In fact, the issue of reduction of dependence space $(U, \Gamma(\mathcal{T}))$ has been discussed detailedly in \cite{ZhangLiangWu03Information}.

\begin{lemma}\cite{ZhangLiangWu03Information}
\label{L:reductofandependencesapce}
$B(\subseteq U)$ is a reduct of $(U, \Gamma(\mathcal{T}))$ if and only if $B \in Min(\{D^{'} \subseteq U: D \bigcap D^{'} \neq \emptyset ~(\forall D \in \overline{\mathcal{T}})\})$, where $\overline{\mathcal{T}} = \{D \neq \emptyset, U - D \in \mathcal{T}\}$.
\end{lemma}

Therefore, we can obtain the following result.
It indicates that a subset of a finite nonempty set is a reduct of the dependence space induced by the coatoms of a geometric lattice if and only if it is a minimal set with respect to the property of containing an element from each nonempty complement of any coatom of the lattice.
The symbol $Com$ appearing in the proposition below is defined as $Com(\mathcal{A}) = \{X \subseteq U: U - X \in \mathcal{A}\}$, where $\mathcal{A}$ is a family subsets of $U$.

\begin{proposition}
\label{P:thereductofthedependencespaceinducedbyamatroid}
$B$ is a reduct of $(U, \Gamma(\mathcal{H}(M)))$ if and only if $B \in Min(\{B \subseteq U: B \bigcap C \neq \emptyset~(\forall C \in Com(\mathcal{H(M)}))\})$.
\end{proposition}

\begin{proof}
According to the definition of hyperplane, we know $U \notin \mathcal{H}(M)$.
It implies that $\emptyset \notin Com(\mathcal{H(M)}))$.
Combining Proposition \ref{P:theotherformoftherelationonmatroid} and Lemma \ref{L:reductofandependencesapce}, we obtain the result.
\end{proof}

\begin{example}
Suppose lattice is the one shown in Example \ref{E:example5}.
Then the coatoms $\mathcal{H}(M)$ of the lattice is $\{\{1,2\}, \{1,3\}, $ $\{2,3\}, $ $\{1,4,5\}, \{2,4,5\}, \{3,4,5\}\}$ and $Com(\mathcal{H}(M))$ $ =\{\{3, 4, 5\}, \{2, 4,$ $ 5\}, \{1, 4, 5\}, \{2, 3\}, \{1, 3\}, \{1, 2\}\}$.
They are all nonempty sets.
According to Proposition \ref{P:thereductofthedependencespaceinducedbyamatroid}, the set of all the reducts of $(U, \mathcal{T}(\mathcal{H}(M)))$ is $\{\{1, 2, 3\},$ $ \{1, 2,$ $ 4\}, \{1, 2, 5\}, \{1, 3, 4\}, \{1, 3, 5\}, \{2, 3, 4\}, \{2, 3, 5\}\}$.
\end{example}

Considering that geometric lattices have a closed relation with matroids, we define the other dependence space from the viewpoint of matroids.
It is interesting that the dependence space is equal to the one $(U, \Gamma(\mathcal{H}(M)))$, which provides us the other approach to realize the purpose to apply the geometric lattice to attribute reduction in subsection \ref{S:Anapplicationtoinformationsystems}.

\begin{definition}
\label{D:anequialencerelationinducedbymatroid}
Let $M$ be a matroid on $U$.
One can define an equivalence relation on $2^{U}$ as follows: For all $B, C \subseteq U$,
\begin{center}
    $B \Theta_{M} C \Leftrightarrow cl_{M}(B) = cl_{M}(C)$.
\end{center}
\end{definition}

\begin{lemma}\cite{HuangZhu13Dependence}
\label{L:theotherclosureexpressionbyhyperplane}
Let $M$ be a matroid on $U$.
For all $X \subseteq U$,
\begin{align}
cl_{M}(X) = \left\{\begin{aligned}%
&U && \mbox r_{M}(X) = r_{M}(U),\\
&\bigcap \{H \in \mathcal{H}(M): X \subseteq H\} && \mbox r_{M}(X) \neq r_{M}(U).\\
\end{aligned}\right.
\end{align}
\end{lemma}

\begin{proposition}
\label{P:theotherformoftherelationonmatroid}
Let $M$ be a matroid on $U$. $(U, \Theta_{M})$ is a dependence space and $\Gamma(\mathcal{H}(M))$ $ = \Theta_{M}$.
\end{proposition}

\begin{proof}
If $(B_{1}, B_{2}) \in \Gamma(\mathcal{H}(M))$, then $B_{1} \subseteq H \Leftrightarrow B_{2} \subseteq H$ for all $H \in \mathcal{H}(M)$.
We know $r_{M}(B_{1}) \neq  r_{M}(U) \neq r_{M}(B_{2}) $.
According to Lemma \ref{L:theotherclosureexpressionbyhyperplane}, we have $cl_{M}(B_{1}) = \bigcap \{H \in \mathcal{H}(M): B_{1} \subseteq H\} = \bigcap \{H \in \mathcal{H}(M): B_{2} \subseteq H\} = cl_{M}(B_{2})$.
Thus $(B_{1}, B_{2}) \in \Theta_{M}$ which implies that $\Gamma(\mathcal{H}(M)) \subseteq \Theta_{M}$.
If $(B_{1}, B_{2}) \in \Theta_{M}$, then $cl_{M}(B_{1}) = cl_{M}(B_{2})$.
For all $H \in \mathcal{H}(M)$, if $B_{1} \subseteq H$, then $cl_{M}(B_{1}) \subseteq cl_{M}(H) = H$.
thus $B_{2} \subseteq cl_{M}(B_{2}) = cl_{M}(B_{1}) \subseteq H$.
Similarly, we can prove the result: For all $H \in \mathcal{H}(M)$, if $B_{2} \subseteq H$, then $B_{1} \subseteq H$.
Therefore, $\Theta_{M} \subseteq \Gamma(\mathcal{H}(M))$.
According to Lemma \ref{L:dependencespace}, we know $(U, \Theta_{M})$ is a dependence space.
\end{proof}
\subsection{An application to information systems}
\label{S:Anapplicationtoinformationsystems}
In subsection \ref{Sub:Geometriclatticeandmatroidalapproachestoreductionissuesindependencespaces}, we propose two methods to solve the problems of reduction in dependence spaces from matroids and geometric lattices, respectively.
In this subsection, we apply the methods to information systems.
First, we introduce the concept of information systems.

\begin{definition}(Information system~\cite{ZhangLiangWu03Information})
An information system is a quadruple form $(U, A,$ $ F, V)$, where $U = \{x_{1}, x_{2}, \cdots, x_{n}\}$ is a nonempty finite set of objects, $A = \{a_{1}, a_{2},$ $ \cdots, a_{m}\}$ is a nonempty finite set of attributes, $V_{j} \in V = \{V_{1}, V_{2}, \cdots, V_{m}\}$ is the domain of attribute $a_{j}$ and $F = \{f_{j}: j \leq m\}$ is a set of information function such that $f_{j}(x_{i}) \in V_{j}$ for all $x_{i} \in U$.
\end{definition}

In an information system, $F$, which describes the connection between $U$ and $A$, is a basis for knowledge discovery.
Here, we assume the information system is complete.
Let $(U, A, F, V)$ be an information system.
For any $B \subseteq A$, the indiscernibility relation is defined as
\begin{center}
    $R_{B} = \{(x_{i}, x_{j}) \in U \times U: f_{l}(x_{i}) = f_{l}(x_{j}), \forall a_{l} \in B\}$.
\end{center}
Specifically, for any attribute $b \in A$,
\begin{center}
    $R_{b} = \{(x_{i}, x_{j}) \in U \times U:f_{b}(x_{i}) = f_{b}(x_{j})\}$.
\end{center}
It is obvious that $R_{B} = \bigcap_{b \in B}R_{b}$ and $R_{B}$, $R_{b}$ are equivalence relations of $U$.
Based on above two equivalence relation, we have the following two equivalence relations:
\begin{center}
    $R = \{(B_{1}, B_{2}) \in P(A) \times P(A): R_{B_{1}} = R_{B_{2}}\}$.\\
\end{center}
and
\begin{center}
$R_{0} = \{(b,c) \in A \times A, R_{b} = R_{c}\}$.
\end{center}

It was noted in \cite{ZhangLiangWu03Information} that $R$ is an equivalence relation on $A$ and the pair $(A, R)$ is a dependence space.
In an information system, $B$ is referred to as a consistent set if $R_{B} = R_{A}$,
and if $B$ is a consistent set and $R_{B - \{b\}} \neq R_{A}$ ($\forall b \in B$), then $B$ is referred to as a reduct of the information system.
We find that the reducts defined in the information system is the reducts defined in the dependence space $(A, R)$.
In the following discussion, we solve the issues of attribute reduction of information systems starting with the operator $R^{\ast}_{0}$.
As we know, the upper approximation operator $R^{\ast}_{0}$ is a closure operator of a matroid.
Similar to Definition \ref{D:anequialencerelationinducedbymatroid}, we have the following equivalence relation.

\begin{definition}
Let $A$ be a finite nonempty set.
For all $X, Y \subseteq A$, one can define an equivalence relation $\Theta$ of $2^{A}$ as follows:
\begin{center}
    $X\Theta Y \Leftrightarrow R^{\ast}_{0}(X) = R^{\ast}_{0}(Y)$.
\end{center}
\end{definition}

According to Proposition \ref{P:theotherformoftherelationonmatroid}, we know $(A, \Theta)$ is a dependence space,
and we can obtain all the reducts of the space through Proposition \ref{P:thereductofthedependencespaceinducedbyamatroid}.
Next, we want to find out all the reducts of $(A, R)$ with the aid of the space $(A, \Theta)$.
The proposition below establishes the relation between $\Theta$ and $R$.

\begin{proposition}
\label{P:relationbetwentwoequivalencerelation}
For all $X, Y \subseteq A$, if $R_{X} = R_{Y} \Rightarrow R^{\ast}_{0}(X) = R^{\ast}_{0}(Y)$, then $\Theta = R$.
\end{proposition}

\begin{proof}
We need to prove $R^{\ast}_{0}(X) = R^{\ast}_{0}(Y) \Rightarrow R_{X} = R_{Y}$.
If $R_{X} \neq R_{Y}$, then we may as well suppose there exists $(x_{i}, x_{j}) \in R_{X} - R_{Y}$.
Thus for all $a \in X$, $(x_{i}, x_{j}) \in R_{a}$ and there exists $b \in Y$ such that $(x_{i}, x_{j}) \notin R_{b}$.
Consequently, $R_{b} \neq R_{a}$ for all $a \in X$.
That implies that $b \notin R^{\ast}_{0}(X) =\{x \in A, [x]_{R_{0}} \bigcap X \neq \emptyset\}$.
It clear that $b \in R^{\ast}_{0}(Y)$ because $b \in Y$ and $b \in [b]_{R_{0}}$.
Therefore, $R^{\ast}_{0}(X) \neq R^{\ast}_{0}(Y)$, a contradiction!
Hence we have $R^{\ast}_{0}(X) = R^{\ast}_{0}(Y) \Rightarrow R_{X} = R_{Y}$.
According to the assumption, we have $R_{X} = R_{Y} \Rightarrow R^{\ast}_{0}(X) = R^{\ast}_{0}(Y)$.
Therefore $\Theta = R_{0}$.
\end{proof}

When an information system satisfies the condition presented in Proposition \ref{P:relationbetwentwoequivalencerelation}, then we can find and prove a method to attribute reduction of the information system.
The method is described as follows: Arbitrarily select an element in each $P_{i}(\in A/R_{0})$ to compose a new set, which is just the reduct of the information system.

\begin{proposition}
\label{reductofinformationsystem}
Let $(U, A, F, V)$ be an information system and $A/R_{0} = \{P_{1}, P_{2}, \cdots,\\ P_{s}\}$.
For all $X, Y \subseteq A$, if $R_{X} = R_{Y} \Rightarrow R^{\ast}_{0}(X) = R^{\ast}_{0}(Y)$, then the following condition holds:
$B$ is a reduct of $(U, A, F, V)$ if and only if $B = \{p_{1}, p_{2}, \cdots, p_{s}\}$, where $p_{i} \in P_{i}$, $1 \leq i \leq s$.
\end{proposition}

\begin{proof}
Suppose $R^{\ast}_{0}$ is a closure operator of matroid $M$.
It is clear that $\mathcal{H}(M) = \{A - P_{i} : i = 1, 2, \cdots, s\}$.
Combining with Proposition \ref{P:thereductofthedependencespaceinducedbyamatroid} and \ref{P:relationbetwentwoequivalencerelation}, we have $B$ is a reduct of $(U, A, F, V)$ if and only if $B \in Min(\{B \subseteq A: B \bigcap C \neq \emptyset(\forall C \in Com(\mathcal{H}(M)))\}) = Min(\{B \subseteq A: B \bigcap C \neq \emptyset (\forall C \in A/R_{0})\}) = \{B \subseteq A: |B \bigcap P_{i}| = 1, i = 1, 2, \cdots, s\}$.
\end{proof}

A relation table entirely determines an information system.
The following example presents how to use above results to find all the reducts of an information system.

\begin{example}
Let $(U, A, F, V)$ be an information system which is shown in table 1.
\begin{table}[h]
\caption{An information system}
\begin{center}
    \begin{tabular}{c|ccc}
  \hline
  ~~~~~U &~~~~~~~~~$Outlook~(a_{1})$ &~~~~~~~~~ $Temperature~(a_{2})$ &~~~~~~~~~ $Humidity~(a_{3})$ ~~~~~\\
  \hline
  ~~~~$x_{1}$ &~~~~~~~~~ sunny &~~~~~~~~~ hot &~~~~~~~~~ high~~~ \\
  ~~~~$x_{2}$ &~~~~~~~~~ rain &~~~~~~~~~ mild &~~~~~~~~~ normal ~~~\\
  ~~~~$x_{3}$ &~~~~~~~~~ rain &~~~~~~~~~ cool &~~~~~~~~~ normal ~~~\\
  ~~~~$x_{4}$ &~~~~~~~~~ rain &~~~~~~~~~ hot &~~~~~~~~~ normal ~~~\\
  \hline
\end{tabular}
\end{center}
\end{table}
It is obvious that $U/R_{a_{1}} = U/R_{a_{3}} = U/R_{\{a_{1}, a_{3}\}} = \{\{x_{1}\}, \{x_{2}, x_{3}, x_{4}\}\}$, $U/R_{a_{2}} = \{\{x_{1}, x_{4}\}, \{x_{2}\}, \{x_{3}\}\}$ and $U/R_{\{a_{1}, a_{2}\}} = U/R_{\{a_{2}, a_{3}\}} = U/R_{\{a_{1}, a_{2}, a_{3}\}} = \{\{x_{1}\},$ $ \{x_{2}\}, \{x_{3}\}, \{x_{4}\}\}$.
Then $R_{0} = \{\{a_{1}, a_{3}\}, \{a_{2}\}\}$.
It is easy to check that any two subsets $X$ and $Y$ of $A$ satisfy the condition: $R_{X} = R_{Y} \Rightarrow R^{\ast}_{0}(X) = R^{\ast}_{0}(X)$.
Hence we can obtain all the reducts of $(U, A, F, V)$ are $\{a_{1}, a_{2}\}$ and $\{a_{2}, a_{3}\}$.
\end{example}
\section{Conclusions}
\label{S:conclusions}
In this paper, we have constructed a geometric lattice from a covering through the transversal matroid induced by the covering, and have used atoms of the lattice to characterize the lattice.
Furthermore, we have applied the lattice to the attribute reduction issues of information systems.
Though some works have been studied in this paper, there are also many interesting topics deserving further investigation.
In the future, we will study algorithm implementations of the attribute reduction issues in information systems through geometric lattices.

\section{Acknowledgments}
This work is supported in part by the National Natural Science Foundation of China under Grant Nos. 61170128, 61379049, and 61379089, 
the Natural Science Foundation of Fujian Province, China, under Grant No. 2012J01294,the Science and Technology Key Project of Fujian Province, 
China, under Grant No. 2012H0043, and the Zhangzhou Research Fund under Grant No. Z2011001.


\begin{thebibliography}{35}
\expandafter\ifx\csname natexlab\endcsname\relax\def\natexlab#1{#1}\fi
\providecommand{\bibinfo}[2]{#2}
\ifx\xfnm\relax \def\xfnm[#1]{\unskip,\space#1}\fi
\bibitem[{Birhoff(1995)}]{Birhoff95Lattice}
\bibinfo{author}{G.~Birhoff},
\bibinfo{title}{Lattice Theory},
\bibinfo{publisher}{American Mathematical Society},
\bibinfo{address}{Rhode Island},
\bibinfo{year}{1995}.

\bibitem[{Davey and Priestley(1990)}]{DaveyPriestley90Introduction}
\bibinfo{author}{B.A~Davey},
\bibinfo{author}{H.A~Priestley},
\bibinfo{title}{Introduction to Lattices and Order},
\bibinfo{publisher}{Cambridge University Press},
\bibinfo{year}{1990}.

\bibitem[{Du et~al.(2011)Du, Hu, Zhu and Ma}]{DuHuZhuMa11Rule}
\bibinfo{author}{Y.~Du},
\bibinfo{author}{Q.~Hu},
\bibinfo{author}{P.~Zhu},
\bibinfo{author}{P.~Ma},
\bibinfo{title}{Rule learning for classification based on neighborhood covering reduction},
\bibinfo{journal}{Information Sciences}
\bibinfo{volume}{181} (\bibinfo{year}{2011})
\bibinfo{pages}{5457--5467}.

\bibitem[{Edmonds(1971)}]{Edmonds71Matroids}
\bibinfo{author}{J.~Edmonds},
\bibinfo{title}{Matroids and the greedy algorithm},
\bibinfo{journal}{Mathematical Programming}
\bibinfo{volume}{1}(\bibinfo{year}{1971})
\bibinfo{pages}{127--136}.

\bibitem[{Estaji et~al.(2012)Estaji, Hooshmandasl and Davva}]{EstajiHooshmandaslDavva12Rough}
\bibinfo{author}{A.~Estaji},
\bibinfo{author}{M.~Hooshmandasl},
\bibinfo{author}{B.~Davva},
\bibinfo{title}{Rough set theory applied to lattice theory},
\bibinfo{journal}{Information Sciences}
\bibinfo{volume}{200} (\bibinfo{year}{2012})
\bibinfo{pages}{108--122}.


\bibitem[{Huang and Zhu(2012)}]{HuangZhu13Dependence}
\bibinfo{author}{A.~Huang},
\bibinfo{author}{W.~Zhu},
\bibinfo{title}{Dependence space of matroids and its application to attribute reduction},
\bibinfo{journal}{arXiv:1312.4231 [cs.AI]}
(\bibinfo{year}{2013}).

\bibitem[{Huang and Zhu(2012)}]{HuangZhu12Geometric}
\bibinfo{author}{A.~Huang},
\bibinfo{author}{W.~Zhu},
\bibinfo{title}{Geometric lattice structure of covering-based rough sets through matroids},
\bibinfo{journal}{Journal of Applied Mathematics}
\bibinfo{volume}{2012} (\bibinfo{year}{2012}),
\bibinfo{pages}{Article ID 236307, 25 pages}.

\bibitem[{Huang et~al.(2013)Huang, Zhao and Zhu(2013)}]{HuangZhaoZhu13Nullity}
\bibinfo{author}{A.~Huang},
\bibinfo{author}{H.~Zhao},
\bibinfo{author}{W.~Zhu},
\bibinfo{title}{Nullity-based matroid of rough sets and its application to attribute reduction},
\bibinfo{journal}{to appear in Information Sciences} (\bibinfo{year}{2013}).


\bibitem[{Lai(2001)}]{Lai01Matroid}
\bibinfo{author}{H.~Lai},
\bibinfo{title}{Matroid theory},
\bibinfo{publisher}{Higher Education Press},
\bibinfo{address}{Beijing},
\bibinfo{year}{2001}.

\bibitem[{Lawler(2001)}]{Lawler01Combinatorialoptimization}
\bibinfo{author}{E.~Lawler},
\bibinfo{title}{Combinatorial optimization:networks and matroids},
\bibinfo{publisher}{Dover Publications},
\bibinfo{year}{2001}.

\bibitem[{Li and Liu(2012)}]{LiLiu12Matroidal}
\bibinfo{author}{X.~Li},
\bibinfo{author}{S.~Liu},
\bibinfo{title}{Matroidal approaches to rough set theory via closure operators},
\bibinfo{journal}{International Journal of Approximate Reasoning}
\bibinfo{volume}{53} (\bibinfo{year}{2012})
\bibinfo{pages}{513--527}.

\bibitem[{Liu and Zhu(2008)}]{LiuZhu08TheAlgebraic}
\bibinfo{author}{G.~Liu},
\bibinfo{author}{W.~Zhu}, \bibinfo{title}{The algebraic structures of generalized rough set theory},
\bibinfo{journal}{Information Sciences}
\bibinfo{volume}{178}(\bibinfo{year}{2008})
\bibinfo{pages}{4105--4113}.

\bibitem[{Min et~al.(2011)Min, He, Qian and Zhu}]{MinHeQianZhu11Test}
\bibinfo{author}{F.~Min},
\bibinfo{author}{H.~He},
\bibinfo{author}{Y.~Qian},
\bibinfo{author}{W.~Zhu},
\bibinfo{title}{Test-cost-sensitive attribute reduction},
\bibinfo{journal}{Information Sciences}
\bibinfo{volume}{181}(\bibinfo{year}{2011})
\bibinfo{pages}{4928--4942}.

\bibitem[{Min and Zhu(2012)}]{FanZhu12Attribute}
\bibinfo{author}{F.~Min},
\bibinfo{author}{W.~Zhu},
\bibinfo{title}{Attribute reduction of data with error ranges and test costs},
\bibinfo{journal}{Information Sciences}
\bibinfo{volume}{211}(\bibinfo{year}{2012})
\bibinfo{pages}{48--67}.

\bibitem[{Ouyang et~al.(2010)Ouyang, Wang and Zhang}]{YaoWangZhang10Onfuzzy}
\bibinfo{author}{Y.~Ouyang},
\bibinfo{author}{Z.~Wang},
\bibinfo{author}{H.~Zhang},
\bibinfo{title}{On fuzzy rough sets based on tolerance relations},
\bibinfo{journal}{Information Sciences}
\bibinfo{volume}{180} (\bibinfo{year}{2010})
\bibinfo{pages}{532--542}.

\bibitem[{Oxley(1993)}]{Oxley93Matroid}
\bibinfo{author}{J.G. Oxley},
\bibinfo{title}{Matroid theory},
\bibinfo{publisher}{Oxford University Press},
\bibinfo{address}{New York},
\bibinfo{year}{1993}.

\bibitem[{Pawlak(1982)}]{Pawlak82Rough}
\bibinfo{author}{Z.~Pawlak},
\bibinfo{title}{Rough sets},
\bibinfo{journal}{International Journal of Computer and Information Sciences}
\bibinfo{volume}{11} (\bibinfo{year}{1982})
\bibinfo{pages}{341--356}.

\bibitem[{Qin et~al.(2007)Qin, Gao and Pei}]{QinGaoPei07OnCovering}
\bibinfo{author}{K.~Qin},
\bibinfo{author}{Y.~Gao},
\bibinfo{author}{Z.~Pei},
\bibinfo{title}{On covering rough sets},
in: \bibinfo{booktitle}{Rough Set and Knowledge Technology}, LNCS, pp.
\bibinfo{pages}{34--41}.

\bibitem[{Slowinski and Vanderpooten(2000)}]{SlowinskiVanderpooten00AGeneralized}
\bibinfo{author}{R.~Slowinski},
\bibinfo{author}{D.~Vanderpooten},
\bibinfo{title}{A generalized definition of rough approximations based on similarity},
\bibinfo{journal}{IEEE Transactions on Knowledge and Data Engineering}
\bibinfo{volume}{12} (\bibinfo{year}{2000})
\bibinfo{pages}{331--336}.

\bibitem[{Tang et~al.(2013)Tang, She, Min and Zhu}]{TangSheMinZhu13AMatroidal}
\bibinfo{author}{J.~Tang},
\bibinfo{author}{K.~She},
\bibinfo{author}{F.~Min},
\bibinfo{author}{W.~Zhu},
\bibinfo{title}{A matroidal approach to rough set theory},
\bibinfo{journal}{Theoretical Computer Science}
\bibinfo{volume}{471} (\bibinfo{year}{2013})
\bibinfo{pages}{1--11}.

\bibitem[{Tang et~al.(2012)Tang, She and Zhu}]{TangSheZhu12Matroidal}
\bibinfo{author}{J.~Tang},
\bibinfo{author}{K.~She},
\bibinfo{author}{W.~Zhu},
\bibinfo{title}{Matroidal structure of rough sets from the viewpoint of graph theory},
\bibinfo{journal}{Journal of Applied Mathematics}
\bibinfo{volume}{2012} (\bibinfo{year}{2012})
\bibinfo{pages}{1--27}.

\bibitem[{Wang et~al.(2012)Wang, Chen, He and Hu}]{WangChenHeHu12A}
\bibinfo{author}{C.~Wang},
\bibinfo{author}{D.~Chen},
\bibinfo{author}{Q.~He},
\bibinfo{author}{Q.~Hu},
\bibinfo{title}{A Comparative Study of Ordered and Covering Information Systems},
\bibinfo{journal}{Fundamenta Informaticae}
\bibinfo{volume}{122} (\bibinfo{year}{2012})
\bibinfo{pages}{1--13}.

\bibitem[{Wang et~al.(2012)Wang, Chen, Sun and Hu}]{WangChenSunHu12Communication}
\bibinfo{author}{C.~Wang},
\bibinfo{author}{D.~Chen},
\bibinfo{author}{B.~Sun},
\bibinfo{author}{Q.~Hu},
\bibinfo{title}{Communication between information systems with covering based rough sets},
\bibinfo{journal}{Information Sciences}
\bibinfo{volume}{216} (\bibinfo{year}{2012})
\bibinfo{pages}{17--33}.

\bibitem[{Wang et~al.(2010{\natexlab{a}})Wang, Liu and Cao}]{WangLiuCao10ANew}
\bibinfo{author}{L.~Wang},
\bibinfo{author}{X.~Liu},
\bibinfo{author}{J.~Cao},
\bibinfo{title}{A new algebraic structure for formal concept analysis},
\bibinfo{journal}{Information Sciences}
\bibinfo{volume}{180}(\bibinfo{year}{2010})
\bibinfo{pages}{4865--4876}.

\bibitem[{Wang et~al.(2010)Wang, Zhu and Zhu}]{WangZhuZhu10Structure}
\bibinfo{author}{S.~Wang},
\bibinfo{author}{P.~Zhu},
\bibinfo{author}{W.~Zhu},
\bibinfo{title}{Structure of covering-based rough sets},
\bibinfo{journal}{International Journal of Mathematical and Computer Sciences}
\bibinfo{volume}{6} (\bibinfo{year}{2010})
\bibinfo{pages}{147--150}.

\bibitem[{Wang et~al.(2012{\natexlab{b}})Wang, Zhu, Zhu and Min}]{WangZhuZhuMin12Matroidal}
\bibinfo{author}{S.~Wang},
\bibinfo{author}{Q.~Zhu},
\bibinfo{author}{W.~Zhu},
\bibinfo{author}{F.~Min},
\bibinfo{title}{Matroidal structure of rough sets and its characterization to attribute reduction},
\bibinfo{journal}{Knowledge-Based Systems}
\bibinfo{volume}{36}(\bibinfo{year}{2012}{\natexlab{b}})
\bibinfo{pages}{155--161}.

\bibitem[{Wang et~al.(2013)Wang, Zhu, Zhu and Min}]{WangZhu12Quantitative}
\bibinfo{author}{S.~Wang},
\bibinfo{author}{Q.~Zhu},
\bibinfo{author}{W.~Zhu},
\bibinfo{author}{F.~Min},
\bibinfo{title}{Quantitative analysis for covering-based rough sets through the upper approximation number},
\bibinfo{journal}{Information Sciences}
\bibinfo{volume}{220}(\bibinfo{year}{2013})
\bibinfo{pages}{483--491}.

\bibitem[{Wang and Zhu(2011)}]{WangZhu11Matroidal}
\bibinfo{author}{S.~Wang},
\bibinfo{author}{W.~Zhu},
\bibinfo{title}{Matroidal structure of covering-based rough sets through the upper approximation number},
\bibinfo{journal}{International Journal of Granular Computing, Rough Sets and Intelligent Systems}
\bibinfo{volume}{2} (\bibinfo{year}{2011})
\bibinfo{pages}{141--148}.

\bibitem[{Wang et~al.(2011)Wang, Zhu and Min}]{WangZhuMin11Transversal}
\bibinfo{author}{S.~Wang},
\bibinfo{author}{W.~Zhu},
\bibinfo{author}{F.~Min},
\bibinfo{title}{Transversal and function matroidal structures of covering-based rough sets},
in: \bibinfo{booktitle}{Rough Sets and Knowledge Technology},
pp. \bibinfo{pages}{146--155}.

\bibitem[{Wei et~al.(2007)Wei, Li and Zhang}]{WeiLiZhang07Knowledge}
\bibinfo{author}{L.~Wei},
\bibinfo{author}{H.~Li},
\bibinfo{author}{W.~Zhang},
\bibinfo{title}{Knowledge reduction based on the equivalence relations defined on attribute set and its power set},
\bibinfo{journal}{Information Sciences}
\bibinfo{volume}{177} (\bibinfo{year}{2007})
\bibinfo{pages}{3178--3185}.

\bibitem[{Wille(1982)}]{Wille82Restructuring}
\bibinfo{author}{R.~Wille},
\bibinfo{title}{Restructuring lattice theory: an approach based on hierarchies of concepts},
in: \bibinfo{editor}{I.~Rival}(Ed.),
\bibinfo{booktitle}{Ordered sets},
\bibinfo{publisher}{Reidel, Dordrecht-Boston},
\bibinfo{year}{1982}, pp. \bibinfo{pages}{445--470}.

\bibitem[{Yao and Chen(2004)}]{YaoChen04Rough}
\bibinfo{author}{Y.~Yao},
\bibinfo{author}{Y.~Chen},
\bibinfo{title}{Rough set approximations in formal concept analysis},
in: \bibinfo{booktitle}{Proceedings of 23rd International Meeting of the North American Fuzzy Information Processing Society},
pp. \bibinfo{pages}{73--78}.

\bibitem[{$\dot{Z}$akowski(1983)}]{zakowski83Approximations}
\bibinfo{author}{W.~$\dot{Z}$akowski},
\bibinfo{title}{Approximations in the space $(u, \pi)$},
\bibinfo{journal}{Demonstratio Mathematical}
\bibinfo{volume}{16}(\bibinfo{year}{1983})
\bibinfo{pages}{761--769}.

\bibitem[{Zhang and Liang and Wu(2003)}]{ZhangLiangWu03Information}
\bibinfo{author}{W.~Zhang},
\bibinfo{author}{Y.~Liang},
\bibinfo{author}{W.~Wu},
\bibinfo{title}{Information systems and Knowledge Discovery},
\bibinfo{publisher}{Science Publishing Company},
\bibinfo{address}{Beijing},
\bibinfo{year}{2003}.

\bibitem[{Zhu(2011)}]{Zhu11Covering}
\bibinfo{author}{P.~Zhu},
\bibinfo{title}{Covering rough sets based on neighborhoods: An approach without using neighborhoods},
\bibinfo{journal}{International Journal of Approximate Reasoning}
\bibinfo{volume}{52}(\bibinfo{year}{2011})
\bibinfo{pages}{461--472}.

\bibitem[{Zhu(2007)}]{Zhu07Topological}
\bibinfo{author}{W.~Zhu},
\bibinfo{title}{Topological approaches to covering rough sets},
\bibinfo{journal}{Information Sciences}
\bibinfo{volume}{177}(\bibinfo{year}{2007})
\bibinfo{pages}{1499--1508}.

\bibitem[{Zhu(2009)}]{Zhu09RelationshipBetween}
\bibinfo{author}{W.~Zhu},
\bibinfo{title}{Relationship between generalized rough sets based on binary relation and covering},
\bibinfo{journal}{Information Sciences}
\bibinfo{volume}{179}(\bibinfo{year}{2009})
\bibinfo{pages}{210--225}.

\bibitem[{Zhu and Wang(2003)}]{ZhuWang03Reduction}
\bibinfo{author}{W.~Zhu},
\bibinfo{author}{F.~Wang},
\bibinfo{title}{Reduction and axiomization of covering generalized rough sets},
\bibinfo{journal}{Information Sciences}
\bibinfo{volume}{152}(\bibinfo{year}{2003})
\bibinfo{pages}{217--230}.

\bibitem[{Zhu and Wang(2011)}]{ZhuWang11Matroidal}
\bibinfo{author}{W.~Zhu},
\bibinfo{author}{S.~Wang},
\bibinfo{title}{Matroidal approaches to generalized rough sets based on relations},
\bibinfo{journal}{International Journal of Machine Learning and Cybernetics}
\bibinfo{volume}{2} (\bibinfo{year}{2011})
\bibinfo{pages}{273--279}.

\bibitem[{Zhu and Wang(2013)}]{ZhuWang13Rough}
\bibinfo{author}{W.~Zhu},
\bibinfo{author}{S.~Wang},
\bibinfo{title}{Rough matroids based on relations},
\bibinfo{journal}{Information Sciences}
\bibinfo{volume}{232} (\bibinfo{year}{2013})
\bibinfo{pages}{241--252}.

\end{thebibliography}

\end{document}